\documentclass{article}

 \usepackage[final,nonatbib]{nips_2017}
\usepackage{hyperref}
\usepackage{amsmath,amsthm,amsbsy,amssymb,epsfig,graphicx,mathtools,bbm}

\usepackage{color, subfig}
\usepackage{setspace,multirow, verbatim, amsfonts, graphicx, amsmath, amsthm,amsbsy, amssymb,  epsfig, url,mymath,booktabs}

\newtheorem{theorem}{Theorem}[section]
\newtheorem{lemma}[theorem]{Lemma}

\newtheorem{example}{Example}[section]

\title{Geometric GAN}

\author{
Jae Hyun Lim$^1$, Jong Chul Ye$^{2,3}$\\
$^1$ ETRI, South Korea\\
\texttt{jaehyun.lim@etri.re.kr} \\
$^2$ Dept. of Bio and Brain engineering, KAIST, South Korea\\
$^3$ Dept. of Mathematical Sciences,  KAIST, South Korea\\
 \texttt{jong.ye@kaist.ac.kr}}

\begin{document}

\maketitle

\begin{abstract}
Generative Adversarial Nets (GANs)  represent an important milestone
for effective generative models, which has inspired numerous variants seemingly different from each other.
One of the main contributions of this paper is to reveal a unified geometric structure in  GAN and its variants.
Specifically, we show that the adversarial generative model training can be decomposed into three geometric steps: separating hyperplane search,  discriminator
parameter update away from the separating hyperplane, and the generator update
 along the normal vector direction of the separating hyperplane. This geometric intuition reveals the limitations of the existing approaches
and leads us to propose a new formulation called {\em geometric GAN}  using SVM separating hyperplane that maximizes the margin.
Our theoretical analysis shows that the geometric GAN converges to a Nash equilibrium between the discriminator and generator.
In addition, extensive numerical results show that the superior performance of geometric GAN. 
\end{abstract}

\section{Introduction}

Recently, inspired by the success of the deep discriminative models,
Goodfellow et al \cite{goodfellow2014generative} proposed a novel generative model training method called 
generative adversarial nets (GAN).
GAN is formulated 
as a minimax game between a generative network (generator) that maps a random vector into the
data space and a discriminative network (discriminator) trying to  distinguish the generated samples from real samples.
Unlike the classical generative models such as Variational Auto-Encoders (VAEs) \cite{kingma2013auto},   the  minimax formation of GAN can transfer the success of deep discriminative models
to generative models, resulting in significant improvement   in generative model performance \cite{goodfellow2014generative}. %This has spurred numerous type of GAN algorithms.

Specifically, the original form of the GAN solves the following minmax game:
\begin{eqnarray} \label{eq:GAN}
\min_G\max_D L_{GAN}(D,G)
\end{eqnarray}
where
\begin{eqnarray}
L_{GAN}(D,G) &:= & E_{x \sim P_S} \left[\log D(x)\right] 
+E_{z\sim P_Z}\left[\log(1-D(G(z)) \right], \label{oGAN}
\end{eqnarray}
where $P_S$ is the sample distribution;
$D(x)$ is the discriminator  that  takes $x \in \Xc$  as input and
outputs a scalar between $[0,1]$; $G(z)$ is the generator that  maps a sample $z$ drawn from a
 distribution  $P_Z$ to the input space $\Xc$.
The meaning of \eqref{eq:GAN} is that the generator tries to fool out the discriminator while the discriminator wants to maximize the differentiation power between the true and generated samples.
The authors further showed that the GAN training is indeed to
approximate the minimization of the symmetric Jensen-Shannon divergence \cite{goodfellow2014generative}.
This idea has been generalized by the authors in \cite{nowozin2016f} for all $f$-divergences.
Moreover, Maximum Mean Discrepancy
objective (MMD) for GAN training was also proposed in  \cite{li2015generative,dziugaite2015training}.

It is well-known that the training GAN is difficult.
In particular,  the authors in \cite{arjovsky2017towards}  have identified the following  sources of the difficulties:
1) when the discriminator becomes accurate, the gradient for generator vanishes,
2) a popular fixation using a generator gradient updating with $E_{z\sim P_Z}\left[- \log D(G(z)) \right]$ is unstable because of the singularity at the denominator when the discriminator is accurate.
%\end{itemize}
The main motivation of  Wasserstein GAN (W-GAN) \cite{arjovsky2017wasserstein} was, therefore, to introduce the weight clipping to address the above-described limitations.  
In fact,  Wasserstein GAN  is a special instance of minimizing the integral probability metric (IPM) \cite{muller1997integral}, and Mroueh et al \cite{mroueh2017mcgan} recently generalized the W-GAN for wider function classes and proposed the mean feature matching and/or covariance feature matching GAN (McGAN)  using the IPM minimization framework \cite{mroueh2017mcgan}.

Inspired by McGAN,  here we propose a novel geometric generalization called {\em geometric GAN.}
Specifically, geometric GAN is  inspired by  our  novel observation that McGAN is composed of  three geometric  operations  in {feature space}:
\begin{itemize}
\item {\bf Separating hyperplane search}: finding the separating hyperplane for a linear classifier  \cite{marron2007distance,carmichael2017geometric,ahn2010maximal}
\item {\bf Discriminator update away from the hyperplane: }  discriminator parameter update {\em away from} the separating hyperplane 
using stochastic gradient direction (SGD).
\item {\bf Generator update toward the hyperplane: } generator parameter update  along the normal vector direction of the separating hyperplane
using stochastic gradient direction (SGD). 
\end{itemize}
This geometric interpretation proves to be very general, so it can be applied to most of the existing GAN and its variants.
Indeed, %the  second and third steps for discriminator and generator update are common to all algorithms,
%and
 the main differences between the algorithms come from the construction of the separating hyperplanes for a linear classifier on  feature space and the { geometric scaling factors} for the feature vectors.
Based on this observation,  we provide  new geometric interpretations 
of GAN \cite{goodfellow2014generative}, $f$-GAN \cite{nowozin2016f}, EB-GAN \cite{zhao2017ebgan},  and W-GAN \cite{arjovsky2017wasserstein}  in terms of  separating hyperplanes and geometric scaling factors,  and discuss their limitations.  
Furthermore, we propose a novel {\em geometric
GAN} using the support vector machine (SVM) separating hyperplane that has maximal margin between two classes of separable data \cite{scholkopf2002learning,carmichael2017geometric}.
Our numerical experiments clearly showed that  the proposed geometric GAN  outperforms the existing GANs in all data set.

\section{Related Approaches}\label{sec:theory}

In order to introduce the geometric interpretation of GAN and its variants, we begin with the review
of the mean feature matching GAN (McGAN) \cite{mroueh2017mcgan}.

\subsection{Mean feature matching GAN}

Let $\Fc$ be a set of
bounded real valued functions on the sample space  $\Xc$. 
Suppose that $P$ and $Q$ are two probability distributions on $\Xc$.
Then, the integral probability metric (IPM) between $P$ and $Q$ on the function
space $\Fc$ is defined as follows \cite{muller1997integral,sriperumbudur2010hilbert,sriperumbudur2009integral}:
\begin{eqnarray}\label{eq:IPM}
d_\Fc(P,Q) & =&  \sup_{f\in \Fc} \left| \int f dP - \int f dQ \right| \notag\\
&=&   \sup\limits_{f\in \Fc} \left| E_{x\sim P}[ f(x)] - E_{x \sim Q} [f(x)] \right| 
\end{eqnarray}
where $E_{x\sim P}[\cdot]$ denotes the expectation with respect to the probability distribution $P(x)$.
We can easily show that $d_\Fc$ is non-negative, symmetric and satisfies the triangle
inequality. So  $d_\Fc$ can be used as a distance measure in the probability space.
For example, when $\Fc$ is defined as a collection of functions with a finite Lipschitz constant,
the IPM is the Wasserstein distance or the earth mover's distance \cite{sriperumbudur2010hilbert,sriperumbudur2009integral}
that forms the basis of the Wasserstein GAN \cite{arjovsky2017wasserstein}.

\begin{figure}[!h]
\centering
\includegraphics[width=10cm]{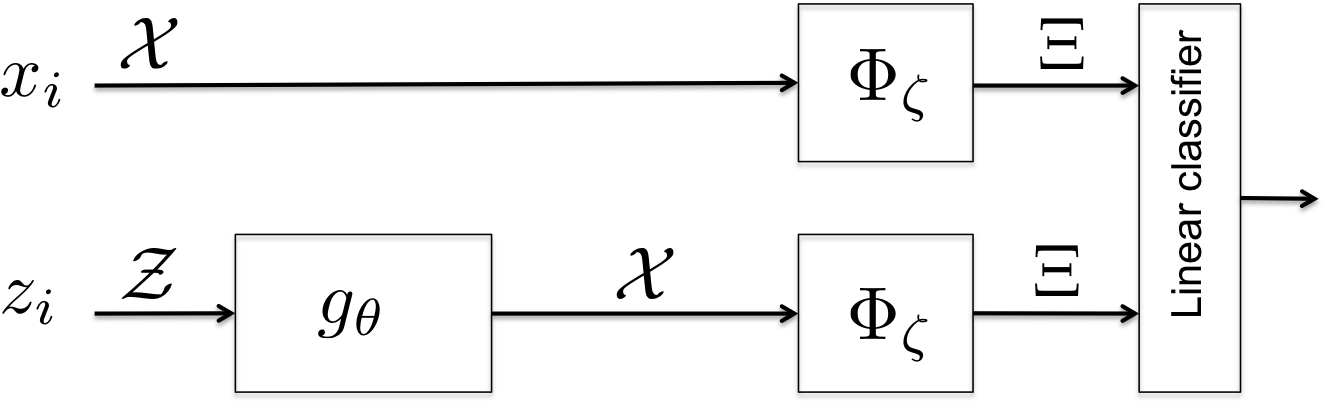}
\caption{Structure of the mean feature matching GAN and its extension to geometric GAN.}
\label{fig:network}
\end{figure}

In McGAN, the  generator network  $g_\theta: \Zc \rightarrow \Xc$  that
maps a random input $z \in \Zc$  to a target $x \in \Xc$  is shown as a block in Fig.~\ref{fig:network}(a), and the
 function space $\Fc$ under study is defined as follows \cite{mroueh2017mcgan}:
$$\Fc_{w,\zeta} = \left\{ f(x) = \langle w, \Phi_\zeta(x) \rangle ~|~ w \in \Xi,  \|w\|_2 \leq 1 \right\}$$
where $ \Phi_\zeta: \Xc \rightarrow \Xi $ is a bounded map  from $\Xc$ to a  (often higher-dimensional) {\em feature space} $\Xi$.
Note that $\Fc_{w,\zeta}$ is  a symmetric function spaces
because if  $f\in \Fc_{w,\zeta}$, then $-f \in \Fc_{w,\zeta}$. Then, the IPM between $P_x$ and $P_g$ is given by
\begin{eqnarray*}
d_\Fc(P_x,P_g) &=&   \max\limits_{\|w\|_2\leq 1} \langle w, E_{x\sim P_x} \Phi_\zeta(x) - E_{g_\theta \sim P_g} \Phi_\zeta(g_\theta) \rangle  %\\
%&=& \max\limits_{\|w\|_2\leq 1} \langle w, E_{x\sim P_x} \Phi(x) - E_{z\sim P_z} \Phi(g_\theta(z)) 
\end{eqnarray*}
under the constraints that $\Phi_\zeta$ is bounded and $\|w\|_2 \leq 1$.

Now, given  a finite sequence of {mini-batch} training data set $S=\{(z_1,x_1),\cdots, (z_n, x_n)\}$, 
 an empirical estimate of the minmax game that minimizes the $d_\Fc(P_x,P_g)$ is given by \cite{mroueh2017mcgan}:
\begin{eqnarray}\label{eq:minmax}
\min_{\theta}\max\limits_{\|w\|_2\leq 1,\zeta} & \hat L(w,\zeta,\theta) \\
\end{eqnarray}
where
$$  \hat L(w,\zeta,\theta) := \left\langle w,  \frac{1}{n} \sum_{i=1}^n \Phi_\zeta(x_i) - \frac{1}{n} \sum_{i=1}^n \Phi_\zeta(g_\theta(z_i)) \right\rangle$$

Then, the  discriminator update can be done using a stochastic gradient descent (SGD) \cite{mroueh2017mcgan}:
\begin{eqnarray}\label{eq:discriminator}
(w, \zeta)  &\leftarrow&  (w, \zeta)  + \eta  \left( \nabla_w  \hat L(w,\zeta,\theta), \nabla_\zeta  \hat L(w,\zeta,\theta) \right) % \sum_{i=1}^n  \left\langle w^{IPM}, \nabla_\theta  \Phi(g_\theta(z_i)) \right\rangle
\end{eqnarray}
where $\eta$ is a learning rate.
The authors in \cite{mroueh2017mcgan} also used the projection onto the unit $l_p$ ball and  weight clipping for $w$ and $\zeta$ updates, respectively, to meet the constraints.
Using the updated $w$, the generator update is then given by  \cite{mroueh2017mcgan}:
%which can be solved using the following stochastic gradient descent (SGD):
\begin{eqnarray}\label{eq:theta0}
\theta  &\leftarrow&  \theta  + \eta   \sum_{i=1}^n  \left\langle w, \nabla_\theta  \Phi(g_\theta(z_i)) \right\rangle/n \ .
\end{eqnarray}

\subsection{Geometric interpretation of the mean feature matching GAN}

The primal form of the  $w$ update in \eqref{eq:discriminator} is geometrically less informative, so 
here we use the Cauchy-Swartz inequality to obtain a closed-form  update for  $w$: 
\begin{eqnarray} \label{eq:w1}
w^* =  c  \sum_{i=1}^n \left( \Phi_\zeta(x_i) - \Phi_\zeta(g_\theta(z_i))\right)/n
\end{eqnarray}
where the constant  $c$ is given by $c = \|\sum_{i=1}^n \left(\Phi_\zeta(x_i) -\Phi_\zeta(g_\theta(z_i))\right)/n \|^{\frac{1}{2}}$.
Given $w^*$, the corresponding discriminator and generator updates are  represented by
\begin{eqnarray}
\zeta  &\leftarrow&  \zeta  + \eta    \sum_{i=1}^n  \langle w^*,   \nabla_\zeta  \Phi_\zeta(x_i) \ - \nabla_\zeta  \Phi_\zeta(g_\theta(z_i))  \rangle/n \label{eq:sgd1} \\  % \sum_{i=1}^n  \left\langle w^{IPM}, \nabla_\theta  \Phi(g_\theta(z_i)) \right\rangle
\theta  &\leftarrow&  \theta  + \eta   \sum_{i=1}^n  \left\langle w^*, \nabla_\theta  \Phi_\zeta(g_\theta(z_i)) \right\rangle/n \label{eq:sgd2}
\end{eqnarray}
Note that the update equations \eqref{eq:w1} to \eqref{eq:sgd2} are equivalent to  the dual form of the McGAN \cite{mroueh2017mcgan}, where the authors derived
the following minmax problem by using \eqref{eq:w1}:
\begin{eqnarray}\label{eq:mintheta}
\min_\theta \max_{\zeta} \frac{1}{2}  \left\|  \sum_{i=1}^n \Phi_\zeta(x_i)  -\sum_{i=1}^n\Phi_\zeta(g_\theta(z_i)) \right\|^2 \ . 
\end{eqnarray}
However, we notice that the explicit representation by Eqs. \eqref{eq:w1} to \eqref{eq:sgd2}  gives clearer geometric intuition that plays the key role in designing a geometric GAN, 
ass will become clear soon. %so we use Eqs. \eqref{eq:w1} to \eqref{eq:sgd2}.

More specifically, in designing a linear classifier for  two class classification problems,
\eqref{eq:w1}  is known as the normal vector for the separating hyperplane for  the {\em mean difference} (MD)  classifier \cite{marron2007distance,carmichael2017geometric,ahn2010maximal}. %  
As shown in Fig.~\ref{fig:geometry},
once the separating hyperplane is defined, the SGD udpate  \eqref{eq:sgd1} is to update the discriminator parameters such that
the true and fake samples are maximally separately away from the separating hyperplane parallel to the normal vector.
On the other hand,  the SGD udpate using \eqref{eq:sgd2} is to update the generator parameters   to make
the fake samples approach   the separating hyperplane along the normal vector direction
(see 
Fig.~\ref{fig:geometry}).
\begin{figure}[!h]
\centering
\includegraphics[width=10cm]{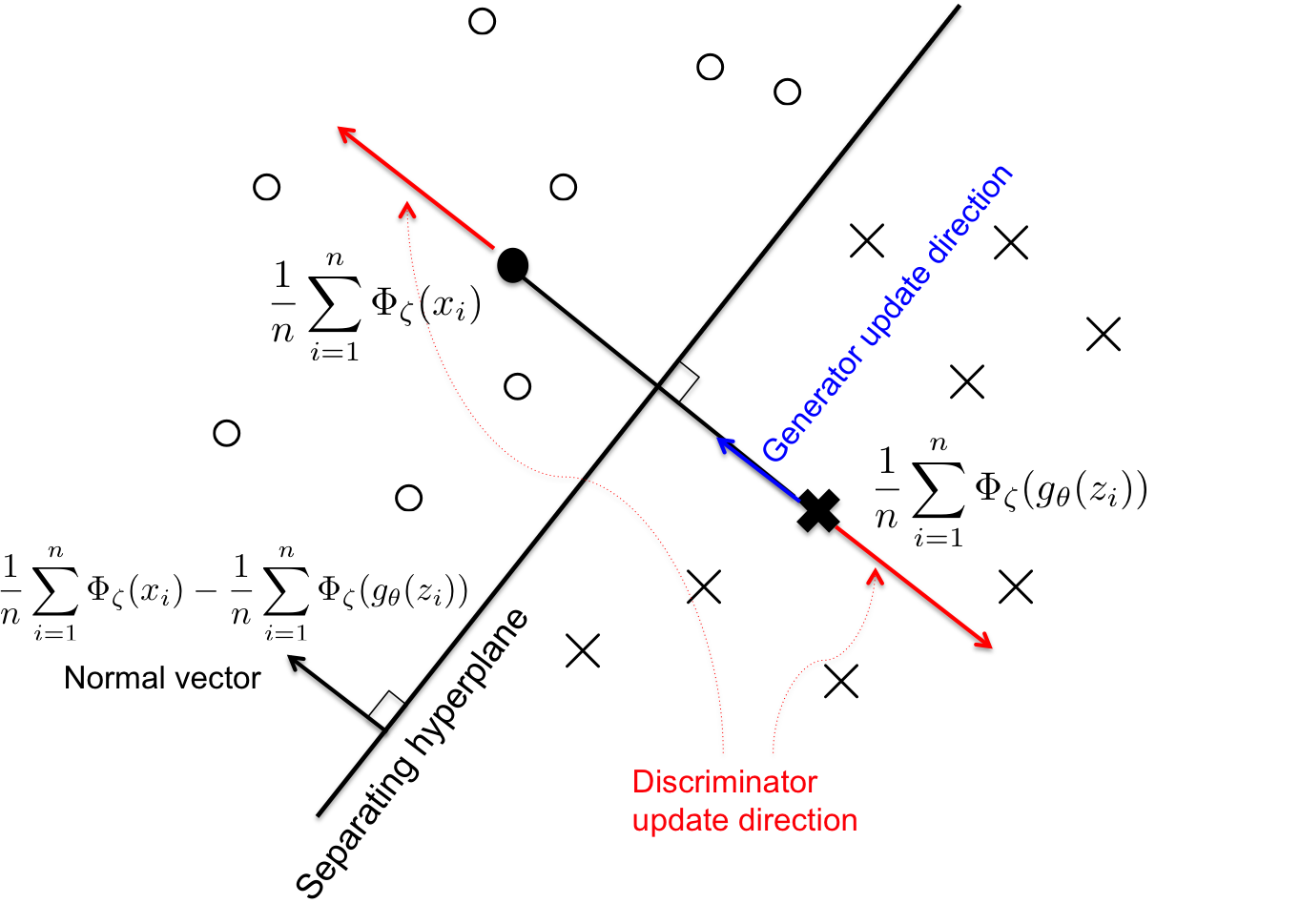}
\caption{Geometry of the mean feature matching GAN.}
\label{fig:geometry}
\end{figure}

Recall that a linear classifier is defined via the normal vector to the separating hyperplane and an offset \cite{scholkopf2002learning}.
Thus, comparing the direction between two classifiers means comparing their normal vector directions.
As shown in Appendix, aside from the geometric scaling factors, 
 the existing GAN and its variants mainly differ in their definition of the
normal vector for the separating hyperplane.
Based on this observation, in the next section,  we discuss an optimal
separating hyperplane for generative model training that has the maximal margin.

\section{Geometric GAN}

\subsection{Linear classifiers in  high-dimension low-sample size}

In adversarial training in feature space, a discriminator is interested in discriminating true samples $\{\Phi_\zeta(x_i)\}_{i=1}^n$ and the fake samples
 $\{\Phi_\zeta(g_\theta(z_i))\}_{i=1}^n$. In practice, the minibatch size $n$ is much smaller than the dimension of the feature space $d$,
 and this type of classification problem is often called the high-dimension low-sample size (HDLSS) problem \cite{marron2007distance,carmichael2017geometric,ahn2010maximal}.
 
In fact, the mean difference (MD) classifier is one of the popular methods for HDLSS.
Specifically, the MD classifier selects the hyperplane that lies half way between the two class means. In
particular the normal vector $w$ for the seperating hyperplane is given by the difference of the class means:
\begin{eqnarray}
w^{MD} &=& \frac{1}{n} \sum_{i=1}^n \Phi_\zeta(x_i) - \frac{1}{n} \sum_{i=1}^n \Phi_\zeta(g_\theta(z_i)) 
\end{eqnarray}
Note that if the variables are first mean centered then scaled by the standard deviation, 
then the mean difference is equivalent to the naive Bayes classifier \cite{marron2007distance,carmichael2017geometric,ahn2010maximal}.
Furthermore, in HDLSS, there always exists a {maximal data piling direction} (MDP) \cite{ahn2010maximal}, where mulitple points in each class have the identical projection on the line spanned by the normal vector.

On the other hand, the Support Vector Machine (SVM) \cite{scholkopf2002learning} and its many variants is one of the most widely used and well studied 
classification algorithms  and its robustness has been also  proven for HDLSS setup  \cite{marron2007distance,carmichael2017geometric,ahn2010maximal}. 
Although the aforementioned classification algorithms are motivated by  fitting a {\em statistical} distribution to the
data,  SVM is motivated by a {\em geometric} heuristic that
leads directly to an optimization problem: maximize the margin between two classes of
separable data. In addition, soft-margin SVM balances two competing objectives to maximize the margin
while penalizing points on the wrong side of the margin.

Recently,  Carmichael et al \cite{carmichael2017geometric} investigated the Karush-Kuhn-Tucker conditions to provide
rigorous mathematical proof for new insights into the behaviour of soft-margin  SVM in the large
and small tuning parameter regimes in HDLSS.  They revealed that for small tuning parameter, if the number of data in two classes are the same (which is the case in our problem),
then the SVM direction becomes exactly the MD direction.
In addition, for sufficiently
large tuning parameter,  the authors showed that soft margin SVM is equivalent to hard margin
SVM if the data are separable, and the hard-margin SVM has  data piling.
Due to this generality of the soft-margin SVM, the proposed geometric GAN is designed based on soft-margin SVM linear classifier.

\subsection{Geometric GAN with  SVM hyperplane}

 Note that soft-margin SVM  is designed by adding a tunning parameter $C$ and slack variables $\xi_i$ which allows points to be on the wrong
side of the margin \cite{scholkopf2002learning}.
 In our problem classifying the true samples versus fake samples,  the primal form of soft-margin SVM can be formulated by
\begin{eqnarray*}
\min_{w,b}  & \frac{1}{2} \|w\|^2 + C \sum_i (\xi_i + \xi_i')
&\\
\mbox{subject to } & \langle w,  \Phi_\zeta(x_i) \rangle+b \geq 1-\xi_i, &  i=1,\cdots, n\\
&  \langle w,  \Phi_\zeta(g_\theta(z_i)) \rangle+b  \leq \xi'_i -1 , &  i=1,\cdots, n\\
& \xi_ i, \xi'_i \geq 0 ,  \quad i=1,\cdots, n
\end{eqnarray*}
Equivalently, the primal form of the soft-margin SVM can be represented using loss + penalty form \cite{hastie2004entire,scholkopf2002learning}:
\begin{eqnarray}\label{eq:optSVM}
\min_{w,b} R_{\theta}(w,b;\zeta)
\end{eqnarray}
where 
\begin{eqnarray}
R_{\theta}(w,b;\zeta) &=&  \frac{1}{2Cn} \|w\|^2  + \frac{1}{n} \sum_{i=1}^n \max\left(0, 1-  \langle w,  \Phi_\zeta(x_i) \rangle- b\right) \notag\\
&& + \frac{1}{n}\sum_{i=1}^n \max\left(0,1+  \langle w,  \Phi_\zeta(g_\theta(z_i)) \rangle + b\right)  \label{eq:SVMcost}
\end{eqnarray}
The goal of the SVM optimization \eqref{eq:optSVM} is to maximize the margin between the two classes.
This implies that the discriminator update can be also easily incorporated with SVM update, because the goal of the discriminator update can be also regarded to
maximize the margin between the two classes. More specifically,  our optimization problem is given by
\begin{eqnarray}\label{eq:optSVM2}
\min_{w,b,\zeta} R_{\theta}(w,b;\zeta)
\end{eqnarray}
for a given generator parameter $\theta$.

% Now, we are interested in investigating the geometry of separating hyperplane and discriminator update.
%
%In fact, the margin between the two margin boundary plays important role in geoemtric GAN, so we define them explicitly:

Specifically, in SVM,  the normal vector for the optimal separating hyperplane $w^{SVM}$ from \eqref{eq:optSVM} is given by \cite{carmichael2017geometric,scholkopf2002learning}:
\begin{eqnarray}\label{eq:wSVM}
w^{SVM}:= \sum_{i=1}^n \alpha_i \Phi_\zeta(x_i) -  \sum_{i=1}^n  \beta_i \Phi_\zeta(g_\theta(z_i)) 
\end{eqnarray}
where  $(\alpha_i,\beta_i)$ 
will be nonzero only for the support vectors, where the set of support vectors now includes all
data points on the margin boundary as well as those on the wrong side of the margin boundary (see Fig.~\ref{fig:geometrySVM}).
%Using the SVM separating hyperplane,
%the feature space of our geometric GAN can be decomposed of three regions  
More specifically, we define the region $\Mc$ between the margin boudaries as shown in Fig.~\ref{fig:geometrySVM}(a):
\begin{eqnarray}
%\Tc &=& \{ \phi \in \Xi  ~|~   \langle w^{SVM},  \phi \rangle + b > 1 \} \\
\Mc &=& \{ \phi \in \Xi  ~|~   |\langle w^{SVM},  \phi \rangle + b |\leq 1 \} \ .
%\Sc &=& \{ \phi \in \Xi  ~|~   \langle w^{SVM},  \phi \rangle + b < -1 \} 
\end{eqnarray}
%where $\Tc$ and $\Sc$ denotes the regions corresponding to the true and fake samples, respectively,
%and $\Mc$ denotes the margin.
%
%
%%Now, let us define the index for the supporting vectors:
%%$$I_{T} = \{ i ~|~ \alpha_i \neq 0\},\quad I_{S} = \{ i ~|~ \beta_i \neq 0\}$$
Then, for  given $w^{SVM}$ and $b$,  the cost function \eqref{eq:SVMcost} then becomes
\begin{eqnarray}
R_{\theta}(w,b;\zeta)  &:=&  \frac{1}{n}\sum_{i\in I_S}   \langle w^{SVM},   \Phi_\zeta(g_\theta(z_i)) \rangle -  \frac{1}{n} \sum_{i\in I_T}   \langle w^{SVM},    \Phi_\zeta(x_i) \rangle  + \mbox{constant} \notag\\
&=& \frac{1}{n} \sum_{i=1}^n  \langle w^{SVM},   s_i\Phi_\zeta(g_\theta(z_i))- t_i \Phi_\zeta(x_i) \rangle + \mbox{constant}
\end{eqnarray}
where $(t_i,s_i)$ are {\em geometric scaling factors}  defined by
\begin{eqnarray}
t_i = \begin{cases} 1, &    \Phi_\zeta(x_i) \in \Mc   \\ 0, & \mbox{otherwise} \end{cases}, &  s_i = \begin{cases} 1, & \Phi_\zeta(g_\theta(z_i)) \in \Mc \\ 0, & \mbox{otherwise} \end{cases}
\end{eqnarray}
This is because the SVM cost function value is not dependent on the feature vectors outside of the margin boundaries
and is now fully determined by the supporting vectors in $\Mc$.
% gives the robustness to the outliers  in adversarial training.
%
%Specifically,  as shown in Fig.~\ref{fig:geometrySVM},   the SVM cost $R_{\theta}$ is

Accordingly,  the discriminator update is given by following SGD updates:
\begin{eqnarray}
\zeta &\leftarrow& \zeta + \eta  \sum_{i=1}^n \langle w^{SVM},  t_i \nabla_\zeta \Phi_\zeta(x_i) - s_i\nabla_\zeta\Phi_\zeta(g_\theta(z_i))/n  \rangle  \label{eq:zetasvm} 
%\theta &\leftarrow& \theta + \eta   \sum_{i=1}^n \langle w^{SVM},   s_i\nabla_\theta\Phi_\zeta(g_\theta(z_i))  \rangle/n   \label{eq:thetasvm}
\end{eqnarray}
In another word,  to  update the discriminator  parameters,  we only need  to push out the supporting vectors toward  the margin boundaries.

On the other hand,  generator update requires more geometric intuition. As shown in Fig.~\ref{fig:geometrySVM}, the generator
update tries to move the fake feature vectors toward the normal vector direction of the separating hyperplane  so that they can be classified as the true feature vectors.
This means that the generator update should be given by the following minimization problem:
$$\min_\theta L_{w,b,\zeta}(\theta)$$
where
\begin{eqnarray}\label{eq:GENcost}
L_{w,b,\zeta}(\theta) &:=&  -\frac{1}{n}\sum_{i=1}^n   D_{w,b,\zeta}(g_\theta(z_i)) \quad  
\end{eqnarray}
with the  linear classifier
\begin{eqnarray}\label{eq:LinD}
 D_{w,b,\zeta}(x):=  \langle w,  \Phi_\zeta(x) \rangle +b    \  .
\end{eqnarray}
This is because  $-D_{w,b,\zeta}(x)$ gets smaller as $  \Phi_\zeta(g_\theta(z))$ moves toward the upper-left side in  Fig.~\ref{fig:geometrySVM} along the normal vector direction $w^{SVM}$.
This results in the following SGD updates:
\begin{eqnarray}
%\zeta &\leftarrow& \zeta + \eta  \sum_{i=1}^n \langle w^{SVM},  t_i \nabla_\zeta \Phi_\zeta(x_i) - s_i\nabla_\zeta\Phi_\zeta(g_\theta(z_i))/n  \rangle  \label{eq:zetasvm} \\
\theta &\leftarrow& \theta + \eta   \sum_{i=1}^n \langle w^{SVM},   \nabla_\theta\Phi_\zeta(g_\theta(z_i))  \rangle/n   \label{eq:thetasvm} \  .
\end{eqnarray}
%where
%$$s_i^G  = \begin{cases} 1, &   \Phi_\zeta(g_\theta(z_i)) \in \Mc \cup \Sc \\  0, & \mbox{otherwise} \end{cases}.$$ 

Note that the SGD updates \eqref{eq:zetasvm} and \eqref{eq:thetasvm} are strikingly similar to \eqref{eq:sgd1} and \eqref{eq:sgd2} of McGAN.
Aside from the different choice of separating hyperplane by \eqref{eq:wSVM}, the discriminator update  \eqref{eq:zetasvm} has additional
geometric scaling factors $(t_i,s_i)$.
As will be shown in Appendix~\ref{ap:variant}, the appearance of the geometric scaling factors is a recurrent theme in geometric interpretation of GAN and its variants, which we believe is fundamental to account for the
geometry of the classifiers.

\begin{figure}[!h]
\centering
\includegraphics[width=10cm]{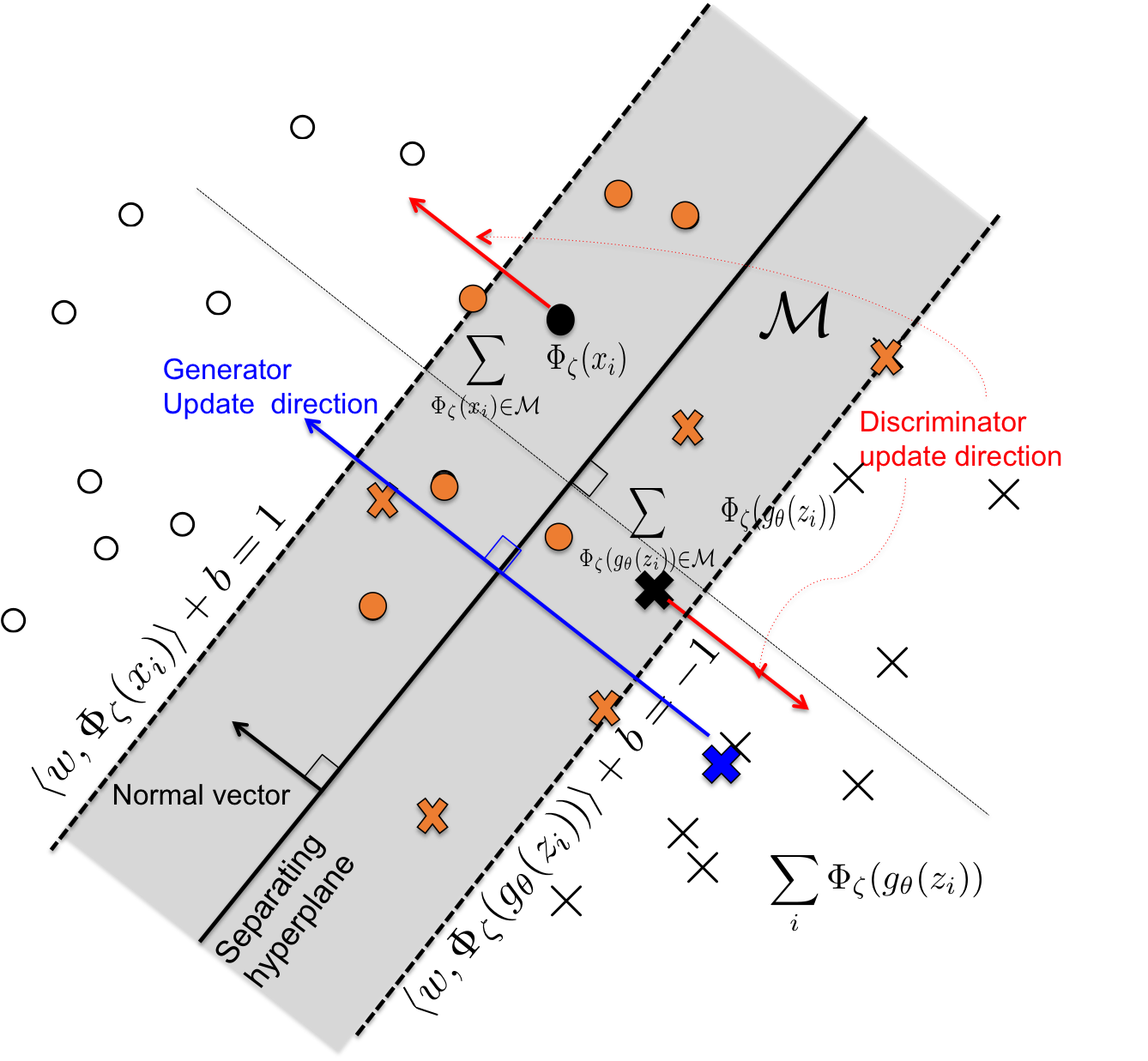} %\includegraphics[width=8cm]{figure/geometrySVM_G.png}
%\centerline{\mbox{(a)}\hspace{8cm}\mbox{(b)}}
\caption{Geometric GAN using SVM hyperplane.  Discriminator and  generator update directions are shown.}
\label{fig:geometrySVM}
\end{figure}

\subsection{Convergence of Geometric GAN}

In order to show the convergence of the geometric GAN to a Nash equilibrium,  we investigate the behaviour at large sample limit.
Specifically, as $n\rightarrow \infty$ for a fixed $C$, the soft margin SVM cost in \eqref{eq:SVMcost} becomes
\begin{eqnarray*}
R(D,g) &% \overset{n\rightarrow \infty}
=  & 
% \frac{1}{2Cn} \|w\|^2  + \frac{1}{n} \sum_{i=1}^n \max\left(0, 1-  \langle w,  \Phi_\zeta(x_i) \rangle- b\right) \notag\\
%& \quad + \frac{1}{n}\sum_{i=1}^n \max\left(0,1+  \langle w,  \Phi_\zeta(g_\theta(z_i)) \rangle + b\right)  \notag\\
 E_{x\sim P_x} \left[ \max\left(0, 1-  D_{w,b,\zeta}(x) \right) \right] 
 + E_{z\sim P_z} \left[ \max\left(0,1+ D_{w,b,\zeta}(g_\theta(z)) \right) \right]  \notag\\
 &=&    \int {\big[ 1 - D_{w,b,\zeta}(x) \big]_{+} dP_x} + \int { \big[ 1 + D_{w,b,\zeta}(g_\theta(z)) \big]_{+} dP_z}\\
&=& \int { p_x(x) \big[ 1 - D_{w,b,\zeta}(x) \big]_{+} dx} + \int {p_{g_\theta}(x) \big[ 1 + D_{w,b,\zeta}(x) \big]_{+} dx}
\end{eqnarray*}
where  $\left[x\right]_{+} = \max\{0, x\}$ and $D_{w,b,\zeta}(x)$ is a linear discriminator in \eqref{eq:LinD} parameterized by $(w,b,\zeta)$.
% as follows:
%$$D_{w,b,\zeta}(x) :=  \langle w,  \Phi_\zeta(x) \rangle + b.$$
 Here, $p_x(x)$ and $p_{g_\theta}(x)$ denote the probability density functions (pdf) for the distribution $P_x$
and $P_z(g_\theta(z))$, respectively.
Similarly, the generator cost function in  \eqref{eq:GENcost} becomes
\begin{eqnarray*}
L(D,g) %&= \frac{1}{n} \sum_{i=1}^n \max\left(0, 1-  \langle w,  \Phi_\zeta(g_\theta(z_i)) \rangle- b\right) \\
&= &- E_{z\sim P_z} \left[ D_{w,b,\zeta}(g_\theta(z))\right] \notag\\
&=&-\int p_{g_\theta}(x) D_{w,b,\zeta}(x) dx
\end{eqnarray*}
Then, the adversarial training between discriminator and generator can be achieved by the following alternating minimization:
\begin{eqnarray}
\min_{D}R(D,g)  &:=& \min_{w,b,\zeta}R(D_{w,b,\zeta},g) \label{eq:nashD} \\
\min_g L(D,g) &:=& \min_\theta L(D, g_\theta) \label{eq:nashG}
\end{eqnarray}

Suppose that  the optimal solution of the aforementioned adversarial training is a pair \( (D^*, g^*) \). Then, we can prove the following key convergence result.

\begin{theorem}\label{thm:conv0}
Suppose that  $(D^*, g^*)$ is a minimizer of the alternating minimization of \eqref{eq:nashD} and \eqref{eq:nashG}.  Then,  $p_{g*}(x) = p_{x}(x)$ almost everywhere, and 
$R(D^*, G^*) =2.$ 
\end{theorem}
\begin{proof}
See Appendix~B.
\end{proof}

%This also implies that the SVM separating hyperplane direction becomes
%that of the maximal data piling direction (MDP) \cite{ahn2010maximal}, where mulitple feature vectors in true (or generator) class have the identical projection along
%the margin boundaries (see Fig.~\ref{fig:MDP}).

%Note that the discriminator cost function \eqref{eq:SVMcost} is differentiable. 
In the following example,  we provide a specific example where the discriminator and
generator cost function has close form expressions, which also has intuitive meaning of the minimum value 2 in Theorem~\ref{thm:conv0}. In particular,
we consider an example of learning parallel lines as in the original Wasserstein GAN paper \cite{arjovsky2017wasserstein}.

\begin{example}[Learning parallel lines]
Let  \(u\sim U[0, 1]\) denote  the uniform distribution on the unit interval.
%, and the true samples are given by 2-D samples given by
%$(0,z)$ where $z \sim U[0,1]$.
Let $P_{x}$ be the distribution of \(x= (0, u) \in \mathbb{R}^{2}\), uniform  on a straight vertical line passing through the origin.
Suppose that the generator sample is given by \(g_\theta(z)=(\theta,z)\) with $\theta$ a single  real parameter.
We can easily see that the SVM separating hyperplane is given by
$$ \langle w, x \rangle +b=0$$
where
$w=(-1, 0), b=\theta/2$ if $\theta \geq 0$ and 
$w=(1, 0), b=-\theta/2$ if  $\theta < 0$.
%\end{align*}
Thus, the generator cost function becomes
\begin{eqnarray*}
L(D,g) &=&  - E\left[\langle w, g_\theta(z) \rangle +b\right] \\
&=& |\theta|/2
\end{eqnarray*}
which achieves its minimum at $\theta^*=0$.
Then, the corresponding  discriminator cost value at  $n\rightarrow \infty$ is given by
\begin{eqnarray*}
R(D^*,g^*)  &=& \lim_{\theta \rightarrow 0} \left\{E \left[1- \langle w, x \rangle -b\right]_+ +  E \left[1+ \langle w, g_\theta(z) \rangle +b\right]_+ \right\} \\
&=& \lim_{\theta \rightarrow 0} 2 \left[1-|\theta|/2\right]_+ \\
&=&  2  %\begin{cases} 2\left[1+\theta/2\right]_+, & \theta\geq 0 \\ 2\left[1-\theta/2\right]_+ , & \theta < 0 \end{cases}
\end{eqnarray*}
which coincides the results by Theorem~\ref{thm:conv0}.
\end{example}

In this example, $R(D^*,g^*)=2$ because all the true and fake samples lies on the separating hyperplane.
This  informs that at the Nash equilibrium of this problem, all the true samples and the fake samples are not separable, which is the desired property of GAN training.
However, Theorem~\ref{thm:conv0} is only a necessary condition to make the true and fakes sample non-separable.
The proof for  the sufficiency condition would be very interesting, which is beyond the scope of current paper. 
%be proven under the current setup.
%$R(D^*,g^*)=2$ does not

\section{Experimental Results}
\subsection{Mixture of Gaussians}
 
%\textcolor{blue} {
In order to evaluate the proposed geometric GAN, we  perform comparative studies with three representative types of GANs; 1) Jenson-Shannon (GAN) \cite{goodfellow2014generative}, 2) mean difference in \(l_\infty\) (Wasserstein GAN) \cite{arjovsky2017wasserstein}, and 3) mean difference in \(l_2\) \cite{mroueh2017mcgan}. 
Here, the behavior of the maximum margin separating hyperplane of the geometric GAN is empirically analyzed against those of the aforementioned approaches. % under the same Lipschitz continuity constraints.
%}

%\textcolor{blue} {
In addition, to evaluate the dependency of each variants on % hyperplanes' property independently with 
Lipschitz continuity constraints, 
the Lipschitz constraints suggested in  \cite{arjovsky2017wasserstein, qi2017lsgan} was also applied to each adversarial training approach.
More specifically, the parameters \( (w, b) \)  of the final linear layer in discriminator  is determined to represent the aforementioned hyperplane properties, whereas the Lipschitz constraints are applied for other network parameters, such as \(\zeta\) in \(\Phi_{\zeta}(x)\) and \(\theta\) in \(g_{\theta}(z)\).
%}
%\textcolor{blue} {
In this paper, we only consider Lipschitz density constraints  in \cite{qi2017lsgan}, so we follow to use weight decay on generators \( g_{\theta}(x) \) and feature space mapping \( \phi_{\zeta}(x) \) in discriminators.
%Although we can apply the Lipschitz continuity constraint only on discriminator function as described in \cite{arjovsky2017wasserstein}, but here we omitted such experiments for simplicity.
%}

%\textcolor{blue} {
We test the four hyperplane searching approaches for discriminators, as well as their complementary generator losses, on two dimensional synthetic data.
The synthetic data consists of 100K data points generated from a mixture of 25 Gaussians, akin to the data that have been used  for describing mode collapsing behaviors of GANs \cite{dumoulin2016ali, metz2016unrolled, che2016moderegulgan}. 
Specifically, the means of the Gaussians are evenly spaced as a 5 by 5 grid along \(x\) and \(y\) axis from -21 to 21.
The standard deviation of each normal distribution is 0.316 (so that the variance would be 0.1).
The sampled data from the true distribution can be seen in Figure \ref{figure:exp1a} and \ref{figure:exp1b}.
%}

%\textcolor{blue} {
For discriminator and generator, a multi-layered fully-connected neural network architecture is used, as described below.
RMSprop \cite{hinton2015rmsprop} is used to train these networks, except vanilla GAN (without any Lipschitz constraints).
For vanilla GAN, Adam \cite{kingma2014adam} with momentum \(\beta_{1}=0.5\) is used.
Base learning rate is set to 0.001.
When weight clipping is applied, parameters in feature mapping \(\phi_{\zeta}(x)\) is clipped within the range of \([-0.01, 0.01]\).
When weight projection on unit \(l_2\) norm is applied, the following rule, \(p=\min\{1, 1/{\Vert p \Vert_{2}}\}\times p \) described in \cite{mroueh2017mcgan} is used to update any parameter \(p\) for every iteration.
For weight decay, weight decaying parameter is set to 0.001.
Batch size is set to 500 for all experiment.
For the number of discriminator update \(K_d\) and the one of generator update \(K_g\), we set them as 1, i.e. \((K_d = 1,  K_g=1)\).
%}

\begin{itemize}
\item {\bf Discriminator:} \(\mathtt{FC(2,128)}\)-\(\mathtt{ReLU}\)-\(\mathtt{FC(128,128)}\)-\(\mathtt{ReLU}\)-\(\mathtt{FC(128,128)}\)-\(\mathtt{ReLU}\)-\(\mathtt{FC(128,1)}\)
\item {\bf Generator:} \(\mathtt{FC(4,128)}\)-\(\mathtt{BN}\)-\(\mathtt{ReLU}\)-\(\mathtt{FC(128,128)}\)-\(\mathtt{BN}\)-\(\mathtt{ReLU}\)-\(\mathtt{FC(128,128)}\)-\(\mathtt{BN}\)-\(\mathtt{ReLU}\)-\(\mathtt{FC(128,2)}\)
\end{itemize}

%\textcolor{blue} {
 The results of the experiment with the mixture of 25 Gaussians are illustrated in Figure \ref{figure:exp1a} and \ref{figure:exp1b}. 
 Amongst all GAN variants in this experiment, geometric GAN demonstrated the least mode collapsing behavior independently with Lipschitz continuity regularization constraints. 
% }
 
%\textcolor{blue} {
 As shown in Fig.~\ref{figure:exp1b}, under the same Lipschitz density constraints, linear hyperplane approaches demonstrated less mode collapsing behaviors by virtue of consistent gradients unlike nonlinear separating hyperplane of original GAN.
 However, mean difference-driven hyperplanes in Wasserstein GAN or McGAN led generators to the mean of arbitrary number of modes in true distributions since the characteristics of mean difference.
 One the other hand, geometric GAN generally showed robust and consistent convergence behavior towards true distributions.   
% }
 
%\textcolor{blue} {
% The Lipchitz density constraints generally improved all algorithms. However, binary classification based approaches, GAN and geometric GAN, showed stronger mode collapsing behaviors when the number of iterations are increased per each step during training, as seen in Figure \ref{figure:exp1b}. 
% }

\begin{figure}
\centering
\begin{minipage}[t][2.45cm]{2.7cm}
  \subfloat[True data]{\includegraphics[width=2.4cm,height=1.8cm]{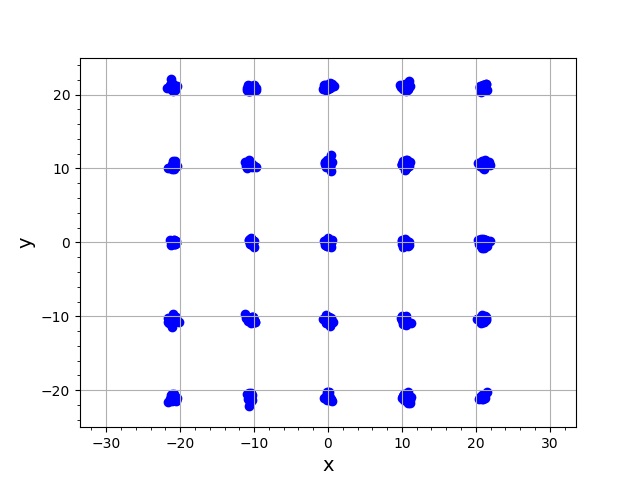}}
\end{minipage}%
\begin{minipage}[t][2.45cm]{2.7cm}
  \subfloat[GAN \cite{goodfellow2014generative}]{\includegraphics[width=2.4cm,height=1.8cm]{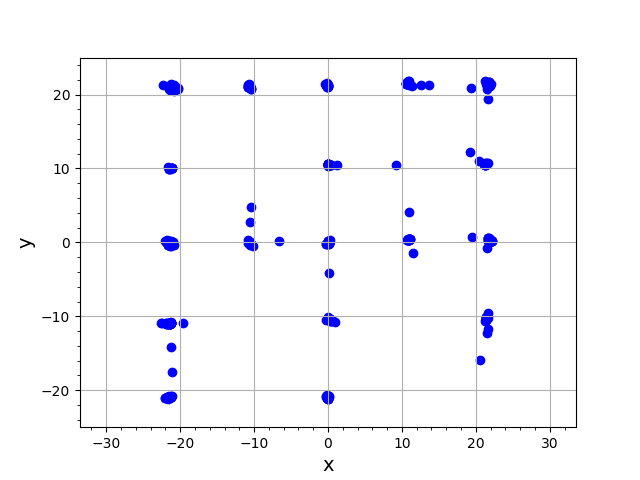}}
\end{minipage}%
\begin{minipage}[t][2.45cm]{2.7cm}
  \subfloat[Proposed]{\includegraphics[width=2.4cm,height=1.8cm]{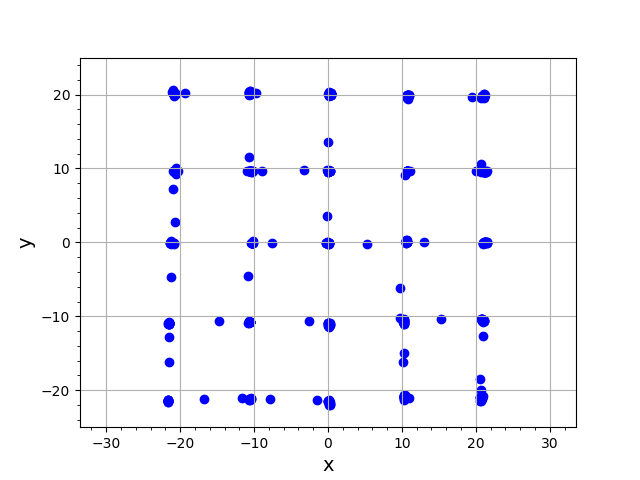}}\end{minipage}%
\begin{minipage}[t][2.45cm]{2.7cm}
  \subfloat[WGAN \cite{arjovsky2017wasserstein}]{\includegraphics[width=2.4cm,height=1.8cm]{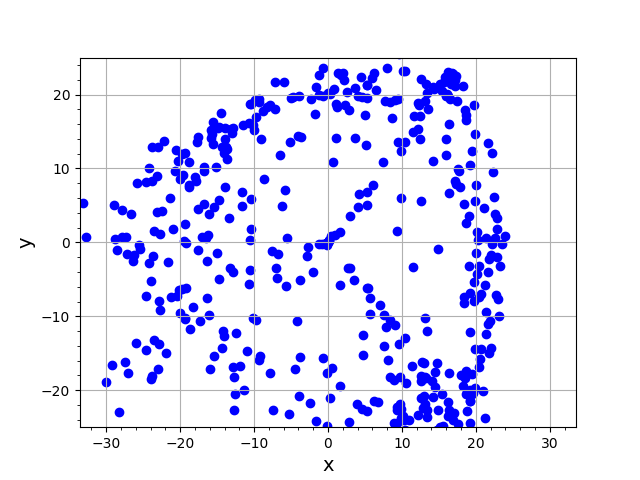}}
\end{minipage}%
\begin{minipage}[t][2.45cm]{2.7cm}
  \subfloat[meanGAN + wproj on \( \phi_{\zeta} \) \cite{mroueh2017mcgan}]{\includegraphics[width=2.4cm,height=1.8cm]{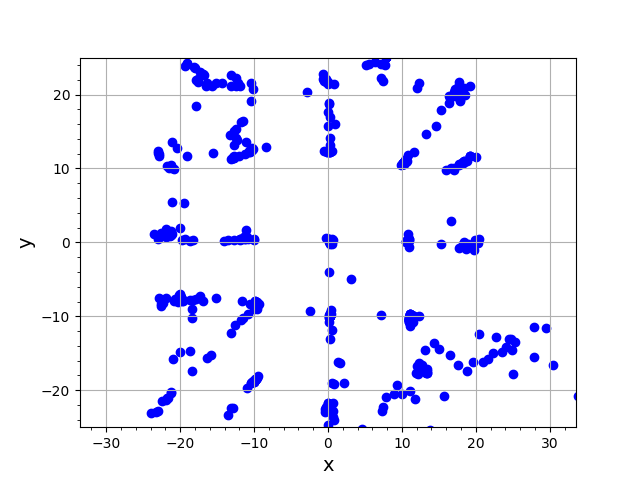}}
\end{minipage}
\caption{Generated samples of GAN variants for the mixture of 25 Gaussians, and the ones from the true data distribution.}
\label{figure:exp1a}
\end{figure}

\begin{figure}
\centering
\begin{minipage}[t][2.75cm]{2.7cm}
  \subfloat[True data]{\includegraphics[width=2.4cm,height=1.8cm]{./real_samples_toy4_gan_toy4_adam_1.png}}
\end{minipage}%
\begin{minipage}[t][2.75cm]{2.7cm}
  \subfloat[GAN + wdecay on \( \Phi_{\zeta} \) and \(g_\theta\)]{\includegraphics[width=2.4cm,height=1.8cm]{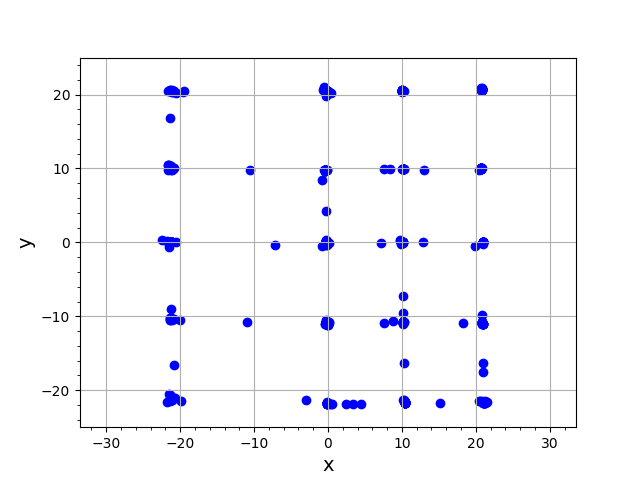}}
\end{minipage}%
\begin{minipage}[t][2.75cm]{2.7cm}
  \subfloat[Proposed + wdecay on \( \Phi_{\zeta} \) and \(g_\theta\)]{\includegraphics[width=2.4cm,height=1.8cm]{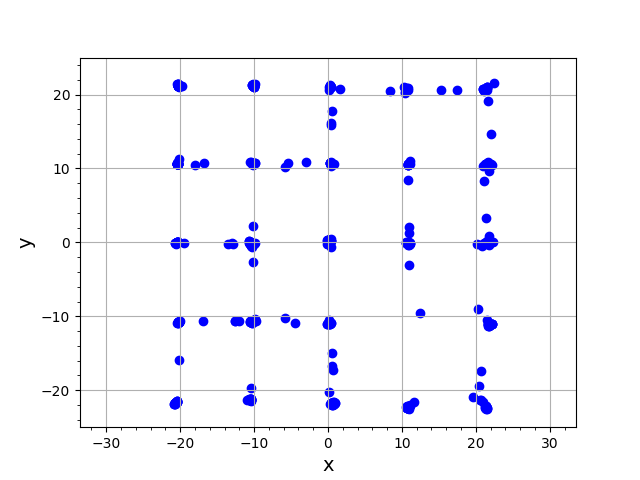}}
\end{minipage}%
\begin{minipage}[t][2.75cm]{2.7cm}
  \subfloat[WGAN + wdecay on \( \Phi_{\zeta} \) and \(g_\theta\)]{\includegraphics[width=2.4cm,height=1.8cm]{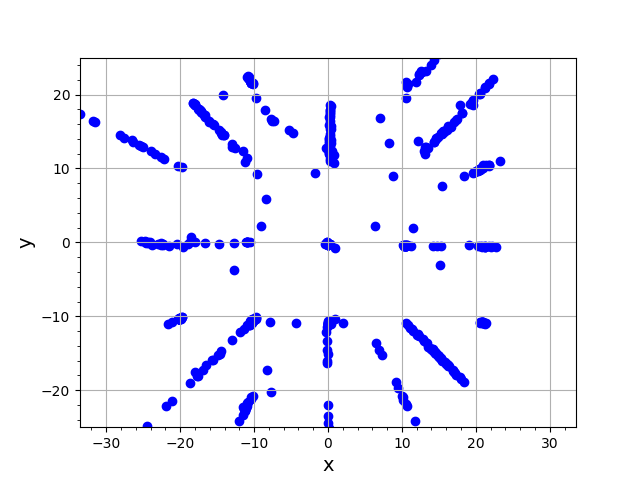}}
\end{minipage}%
\begin{minipage}[t][2.75cm]{2.7cm}
  \subfloat[meanGAN + wdecay on \( \Phi_{\zeta} \) and \(g_\theta\)]{\includegraphics[width=2.4cm,height=1.8cm]{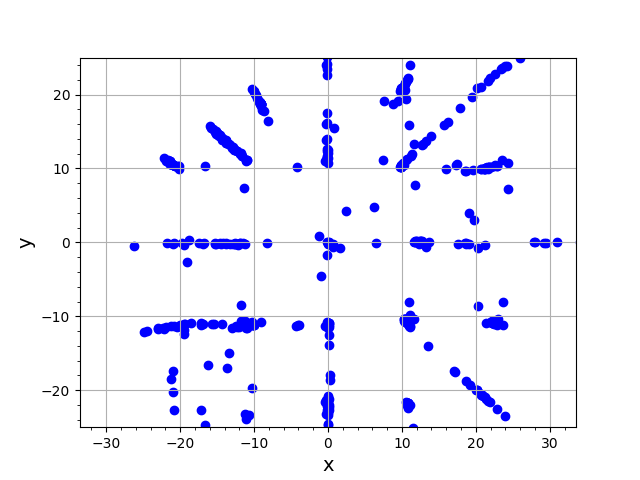}}
\end{minipage}
\caption{Generated samples of GAN variants under trained with Lipschitz density constraints suggested in \cite{qi2017lsgan} for the mixture of 25 Gaussians, and the sample from the true data distribution. During training, weight decay was applied on \(\zeta\) in \(\Phi_{\zeta}(x)\) and \(\theta\) in \(g_{\theta}(z)\) in addition to their own hyperplane constraints.}
\label{figure:exp1b}
\end{figure}

%\subsection{Mixture of Gaussians with Different Tuning Parameters \(C\)}

\subsection{Image Datasets}

%\textcolor{blue} {
In order to analyze the proposed method on large-scale dataset, the geometric GAN is empirically analyzed on well-studied datasets in the context of adversarial training; MNIST, CelebA, and LSUN datasets. 
Since consistent quantitative measures are still under debate, we only perform qualitative comparisons of generated samples from the learned generators of the propsed method against the results of previous literatures.
In favor of fair comparisons with other adversarial training methods, we adopt the settings from the previous literatures except the hyperparameters of stochastic optimizations and the tuning parameter of the proposed method. 
%}

%\textcolor{blue} {
The DCGAN neural network architectures \cite{radford2015dcgan} was used, including batch normalization for generator.
Note that the currently known adversarial training methods that demonstrated stable learning without batch normalization \cite{arjovsky2017wasserstein, mroueh2017mcgan, qi2017lsgan} have resorted to Lipschitz constraints; therefore, it can also be applied to other adversarial training criterions, including geometric GAN, in order to train batch normalization-free generators.
%}

%\textcolor{blue} {
Each pixel value in input image was rescaled to \([-1, 1]\) for all dataset, including MNIST dataset.
During all training, mini-batch size was set to 64. 
For stochastic gradient update during training, RMSprop was used \cite{hinton2015rmsprop}.
The number of generator's updates per each discriminator's update is set to 10 \((K_d = 1, K_g = 10)\). 
%\textcolor{red} {
%Since the lower bound of the Nash equilibrium of the suggested method was depending on the optimality of the generators, it will show stable convergence behavior as long as generator updates are sufficiently applied.??
%}
Learning rate is set to 0.0002 for both discriminators and generators, and a tuning parameter \(C\) for discriminator is set to 1.
%}

%\textcolor{blue} {
Specifically for MNIST dataset, input images were resized to 64 by 64 pixels in order to use the same DCGAN network architecture, and 
the number of epochs for training was set to 20. For CelebA dataset, input images were resized to 96 by 96 pixels and center-cropped with 64 by 64 pixels, and the number of epochs for training is set to 50. For LSUN dataset, only bedroom dataset is used, and an input image is resized to 64 by 64 pixels. The number of epochs for training is set to 2 for LSUN dataset.
%}

The results in Figure ~\ref{figure:exp2:mnist}, \ref{figure:exp3:celeba}, and \ref{figure:exp4:lsun} clearly show that the geometric GAN
generates very realistic images without mode collapsing or divergent behaviours.

%\section{Experimental Results}
%
%\subsection{Mixture of Gaussians}
%
%
%\subsection{Mixture of Gaussians with Different Tuning Parameters \(C\)}
%
%
%
%\subsection{Image Datasets}

\begin{figure}[!hbt]
\centering
\includegraphics[width=10cm]{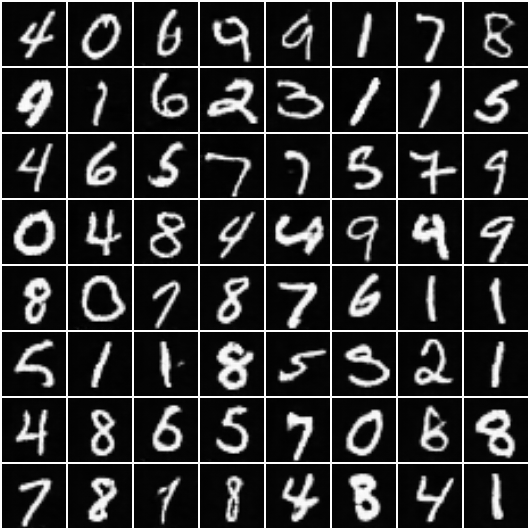}
\caption{Generated samples of Geometric GAN trained for MNIST dataset.}
\label{figure:exp2:mnist}
\end{figure}

\begin{figure}[!hbt]
\centering
\includegraphics[width=10cm]{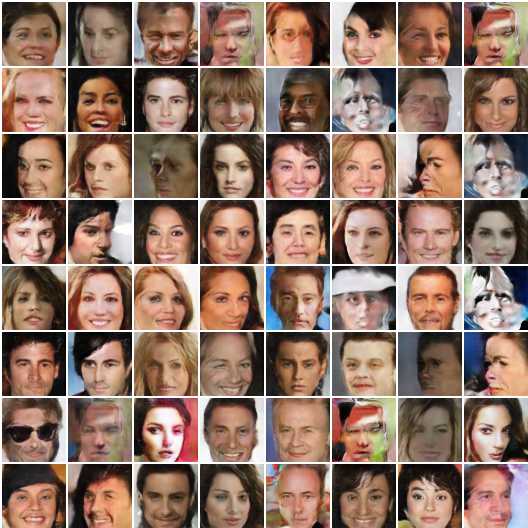}
\caption{Generated samples of Geometric GAN trained for CelebA dataset.}
\label{figure:exp3:celeba}
\end{figure}

\begin{figure}[!hbt]
\centering
\includegraphics[width=10cm]{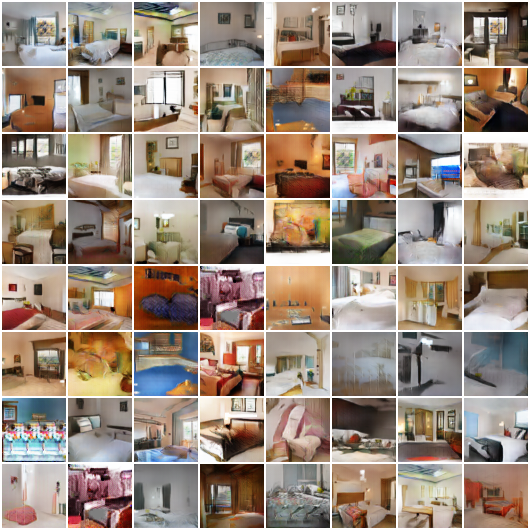}
\caption{Generated samples of Geometric GAN trained for LSUN dataset.}
\label{figure:exp4:lsun}
\end{figure}

\section{Conclusion}
This paper proposed a novel geometric GAN using SVM separating hyperplane, based on geometric intuitions revealed from previous adversarial training approaches. 
The geometric GAN was based on SVM separating hyperplanes that has the maximal margins between the two classes.
Compared to the most of the existing approaches that are based on statistical design criterion, the geometric GAN is derived based on geometric intuition similar to the derivation of SVM.
%\textcolor{red} {
%This paper has suggested a novel geometric GAN using SVM separating hyperplane based on geometric intuitions, unlike currently dominant distance metric-view of adversarial training methods.??
%}
Extensive numerical experiments showed that the proposed method has demonstrated less mode collapsing and  more stable training behavior.
Moreover, our theoretical results showed that the proposed algorithm converges to the Nash equilibrium between the discriminator and generator, which has also geometric meaning.
%It has also showed better convergence behavior to true distribution under the same Lipschitz constraints compared to other known approaches.

\section*{Acknowledgement}
 This work is supported by Korea Science and Engineering Foundation, Grant number NRF2016R1A2B3008104.
 The first author would like to thank 
Yunhun Jang for helpful discussions.

%\bibliographystyle{IEEEtran}
%\bibliography{IEEEabrv,icml}
%%\bibliographystyle{icml2016}


\clearpage

\appendix

\section{Geometric interpretation of GAN  and its variants}
\label{ap:variant}

%The extension of the analysis for GAN variations from IPM minimization (such as Wasserstein GAN)  is so simple that
This appendix provides geometric interpretation  of   GAN and its variants.
In particular, we consider a specific form of the discriminator $D_{w,\zeta}(x_i)$ given by
\begin{eqnarray}\label{eq:Dref}
D_{w,\zeta}(x_i) =  S_f(V_{w,\zeta}(x_i)) ,\quad \mbox{where}\quad V_{w,\zeta}(x_i):=\langle w, \Phi_\zeta(x_i) \rangle 
\end{eqnarray}
where
$S_f$ is an output activation function, 
$V_{w,\zeta}(x_i)$ is the output layer composed of  linear layer  $w$ and the convolutional
neural network below corresponding to $\Phi_\zeta(x)$.
Under this choice of the discriminator, 
we will show that the differences between existing approaches come from the choice of the separating hyperplanes and geometric scaling factors.

\subsection{GAN}

Recall that   the empirical estimate of the GAN cost in \eqref{oGAN} is given by :
$$\hat L_{GAN}(w,\zeta,\theta) :=  \frac{1}{n} \sum_{i=1}^n \log D_{w,\zeta}(x_i)
+\frac{1}{n} \sum_{i=1}^n \log(1-D_{w,\zeta}(g_\theta(z_i)), $$
We now define geometric scaling factors for true and  synthetic (or fake) feature vectors:
$$t_i:=  \frac{S'\left(\langle w, \Phi_\zeta(x_i) \rangle \right)}{D(x_i)}, \quad
\quad s_i:= 
\frac{S'\left(\langle w, \Phi_\zeta(g_\theta(z_i)) \rangle \right)}{1-D(g_\theta(z_i))}$$
In particular, if the activation function is the sigmoid, i.e. $S(u) = 1/(1+e^{-u})$, then we can easily see that
$$t_i:=  {1-D(x_i)}, \quad
\quad s_i:=  D(g_\theta(z_i)).$$

Then, the separating hyperplane update is given by:
\begin{eqnarray}
w^{GAN} & \leftarrow & w^{GAN} + \eta \sum_{i=1}^n \left( t_i\Phi_\zeta(x_i)  -s_i\Phi_\zeta(g_\theta(x_i)) \right)/n
\end{eqnarray}
Using another application  of chain rules,
\begin{eqnarray}
\zeta &\leftarrow& \zeta + \eta  \sum_{i \in I} \langle w^{GAN},  t_i\nabla_\zeta \Phi_\zeta(x_i) - s_i\nabla_\zeta\Phi_\zeta(g_\theta(z_i))  \rangle/n  \label{eq:zetaGAN} \\
\theta &\leftarrow& \theta + \eta  \sum_{i \in I} \langle w^{GAN},  s_i \nabla_\theta\Phi_\zeta(g_\theta(z_i))  \rangle/n \label{eq:thetaGAN}
\end{eqnarray}
%Again, the GAN update rule can be interpreted as three geometric operations as described in geometric GAN.
Aside from different choice of separating hyperplane, the only difference is that the features vectors needs to be scaled appropriated using geometric scaling parameters.
In fact, the scale parameter is directly related to the geometry of the underlying curved feature spaces due to the $\log(\cdot)$ and nonlinear activations.

From \eqref{eq:thetaGAN}, we can easily see that  as discriminator becomes accurate, we have  $s_i  =  D\left(g_\theta(z_i) \right)\simeq 0$, so the update of the generator  becomes more difficult.
This is the main technical limitation of the GAN training.

\subsection{$f$-GAN}

The $f$-GAN formulation is given by the minmax game of the following empirical cost:
\begin{eqnarray*}
F(w,\zeta,\theta) &=& \frac{1}{n} \sum_{i=1}^n S_f(V_{w,\zeta}(x_i)) - \frac{1}{n} \sum_{i=1}^n f^*(S_f(V_{w,\zeta}(g_\theta(z_i)))
\end{eqnarray*}
where  $f^*$ is the convex conjugate of the divergence function $f$.
We again define  a geometric scale factors for true and fake feature vectors:
$$t_i:= S_f'\left(\langle w, \Phi_\zeta(x_i) \rangle \right), \quad
\quad s_i:= 
(f^*)'\left(S_f(\langle w, \Phi_\zeta(x_i) \rangle )\right)S_f'\left(\langle w, \Phi_\zeta(x_i) \rangle \right). $$
The explicit forms of the geometric scaling factors for different $f$-divergences are` summarized in Table~\ref{tab:ts},

\begin{table}[!hbt]
\begin{center}
\begin{tabular}{c|cccccc}
\hline
Name &  $S_f(v)$  &  $f^*(v)$  & $t_i(u)$ & $s_i(u)$ \\
\hline
Total variation  &  $\frac{1}{2} \mathrm{tanh}(v)$  &  $v$ & $\frac{1}{2} \mathrm{coth}(u)$   & $\frac{1}{2} \mathrm{coth}(u)$   &  \\
Kullback-Leiber (KL) &  $v$  &  $\exp(v-1)$   &  1  & $\exp(u-1)$ \\
Reverse KL & $-\exp(v)$ & $-1-\log(-v)$  & $-\exp(u)$ & $1$   \\
Pearson $\chi^2$  &  $v$  & $v^2/4+v$  &  1  & $u/2+1$\\ 
Jensen-Shannon & $\log(2)-\log (1+\exp(-v))$ & $-\log(2-\exp(v))$ & $\frac{e^{-u}}{1+e^{-u}}$  &   $\frac{1}{1+e^{-u}}$ \\
GAN & $-\log (1+\exp(-v))$ & $-\log(1-\exp(v))$ & $\frac{e^{-u}}{1+e^{-u}}$  &  $\frac{1}{1+e^{-u}}$ \\
\hline
\end{tabular}
\caption{Recommended final layer activation functions for $f$-GAN \cite{nowozin2016f} and their geometric scaling factors.}
\label{tab:ts}
\end{center}
\end{table}

Then, the separating hyperplane update is given by:
\begin{eqnarray}
w^{fGAN} & \leftarrow & w^{fGAN} + \eta \sum_{i=1}^n \left( t_i\Phi_\zeta(x_i)  -s_i\Phi_\zeta(g_\theta(x_i)) \right)/n
\end{eqnarray}
From the chain rules,  we have discriminator and generator update rules:
\begin{eqnarray}
\zeta &\leftarrow& \zeta + \eta  \sum_{i \in I} \langle w^{fGAN},  t_i\nabla_\zeta \Phi_\zeta(x_i) - s_i\nabla_\zeta\Phi_\zeta(g_\theta(z_i))  \rangle/n  \label{eq:zetafGAN} \\
\theta &\leftarrow& \theta + \eta  \sum_{i \in I} \langle w^{fGAN},  s_i \nabla_\theta\Phi_\zeta(g_\theta(z_i))  \rangle/n \label{eq:thetafGAN}
\end{eqnarray}

Note that  $f$-GAN is only different from each other in their construction of the weight coefficient $(t_i,s_i)$ (see Table~\ref{tab:ts}) that
reflects the underlying geometry of the curved feature space.
Other than the total variation-based divergence, the scaling factors are asymmetric. Thus, controlling the balance between discriminator and generator updates
are one of the important technical issues of $f$-GAN training.

\subsection{Wasserstein GAN}

Wasserstein GAN \cite{arjovsky2017wasserstein} minimizes the following IPM:
$$d_\Fc(P,Q) = \sup_{\|f\|_L \leq 1}  \frac{1}{n}\sum_{i=1}^n f(x_i) - \frac{1}{n}  \sum_{i=1}^n f(g_\theta(z_i))$$
where 
$\|f\|_L := \sup\{ |f(x)-f(y)|/\rho(x,y): x \neq y  \in M \}$ is called the Lipschitz seminorm of a real-valued function $f$ on $M$.
Using the discriminator model \eqref{eq:Dref}, the Wasserstein GAN update can be written by:
\begin{eqnarray}\label{eq:Wminmax}
\min_{\theta}\max\limits_{\|w\|_\infty\leq 1,\zeta} &  \left\langle w,  \frac{1}{n} \sum_{i=1}^n \Phi_\zeta(x_i) - \frac{1}{n} \sum_{i=1}^n \Phi_\zeta(g_\theta(z_i)) \right\rangle 
\end{eqnarray}
Therefore, other than  the mean difference on $l_\infty$ ball for the hyperplane normal vector  $w$ update,
the W-GAN update is same as the mean matching GAN update with geometric scaling factor $t_i = s_i=1, \forall i$.

\subsection{Energy-based GAN}

For a given a positive margin $m$, %a data sample $x$ and a generated sample $g_\theta(z)$,
the energy-based GAN (EBGAN)  is given by the alternating minimization of the discriminator and generator cost functions \cite{zhao2017ebgan}:
\begin{eqnarray*}
L_D(w,\zeta) &=& \frac{1}{n} \sum_{i=1}^n \left( D_{w,\zeta}(x_i) + \left[m- D_{w,\zeta}(g_\theta(z_i)) \right]_+ \right) \\
L_G(\theta) &=&  \frac{1}{n} \sum_{i=1}^n  D_{w,\zeta}(g_\theta(z_i)) 
\end{eqnarray*}
where  $[x]_+= \max\{0, x\}$. For a function $\psi(y) =a y+b[m-y]_+$ with $y, a,b\geq 0$, its subgradient
is given by:
$$\psi'(y) = \begin{cases} a-b, &  y\in [0,m] \\  a,  &  y \in (m, \infty) \\  [a-b,a], & y \in m \end{cases}$$
Due to the margin,  geometric scale factors for true and fake feature vectors should be defined accordingly.
More specifically, we have % However, this time, the geometric scaling factor for the fake feature
$$t_i:= S_f'\left(\langle w, \Phi_\zeta(x_i) \rangle \right), \quad
\quad s_i^G:= 
S_f'\left(\langle w, \Phi_\zeta(x_i) \rangle \right),\quad s_i:= \begin{cases} s_i,  &  D_{w,\zeta}(g_\theta(z_i)) \in [0,m]  \\ 0, & \mbox{otherwise}\end{cases}$$
%where $s_i^D$ is determined by the margin.
Then, the separating hyperplane update is given by:
\begin{eqnarray}
w^{fGAN} & \leftarrow & w^{fGAN} + \eta \sum_{i=1}^n \left( t_i\Phi_\zeta(x_i)  -s_i\Phi_\zeta(g_\theta(x_i)) \right)/n
\end{eqnarray}
Similarly, we have
\begin{eqnarray}
\zeta &\leftarrow& \zeta + \eta  \sum_{i \in I} \langle w^{fGAN},  t_i\nabla_\zeta \Phi_\zeta(x_i) - s_i\nabla_\zeta\Phi_\zeta(g_\theta(z_i))  \rangle/n \label{eq:EBGANzeta}   \\
\theta &\leftarrow& \theta + \eta  \sum_{i \in I} \langle w^{fGAN},  s_i^G \nabla_\theta\Phi_\zeta(g_\theta(z_i))  \rangle/n  \label{eq:EBGANtheta}
\end{eqnarray}

It is worthy to note that the introduction of {\em margin} appears similar  to our geometric GAN with SVM hyperplane.
In particular,  when a linear activation function is used, we have $t_i = s_i^G=1$, the update equations \eqref{eq:EBGANzeta} and \eqref{eq:EBGANtheta} appears very similar to \eqref{eq:zetasvm}  and
\eqref{eq:thetasvm}, respectively.
However, there exists fundamental differences.
First, in EB-GAN,  only the fake samples outside the margins are  excluded  for the hyperplane and discriminator updates.
On the other hand,  in geometric GAN,  both the true and fake samples outside the margins are excluded for the hyperplane and discriminator updates.
The symmetric exclusion in geometric GAN is observed to make the algorithm more robust to outliers.
Second,  in EBGAN, the margin is defined for the discriminator values. On the other hand,  in geometric GAN,
the margin is determined by the geometric distance between  the feature vectors. 
Therefore, it is much easier to rely on geometric intuition in designing the geometric GAN.

\subsection{Empirical risk minimization}

The empirical risk minimization (ERM) with $l_2$ cost  is one of the standard method for  regression problems.
Although the empirical risk minimization (ERM) is rarely used for generator model,
 our analysis also provides the geometric intuition of ERM update.

Specifically, for a  given mini-batch training data set $S=\{(z_1,x_1),\cdots, (z_n, x_n)\}$,
recall that 
the empirical risk minimization (ERM) in the feature space \cite{scholkopf2002learning} is given by
\begin{eqnarray}\label{eq:ERM}
\min_\theta  \frac{1}{2}  \sum_{i=1}^n \left\|  \Phi_\zeta(x_i)  - \Phi_\zeta(g_\theta(z_i) )\right\|^2
\end{eqnarray}
%Using the triangular inequality, we can easily see that \eqref{eq:ERM}   is equivalent to minimizing the upper bound of  \eqref{eq:mintheta}.
%IPM  between the target and neural network outputs.
Then, the stochastic gradient for \eqref{eq:ERM} can be represented in the identical form  to \eqref{eq:theta0}:
\begin{eqnarray}\label{eq:theta}
\zeta &\leftarrow& \zeta + \eta  \sum_{i \in I} \langle w_i^{ERM},  \nabla_\zeta \Phi_\zeta(x_i) - \nabla_\zeta\Phi_\zeta(g_\theta(z_i))  \rangle/n  \label{eq:zetaERM} \\
\theta  &\leftarrow&  \theta  + \eta   \sum_{i=1}^n  \left\langle w_i^{ERM}, \nabla_\theta  \Phi_\zeta(g_\theta(z_i)) \right\rangle/n \label{eq:thetaERM}
\end{eqnarray}
where $w_i^{ERM}$ is now defined as
\begin{eqnarray} \label{eq:w}
w_i^{ERM}  &=&  \Phi_\zeta(x_i) -  \Phi_\zeta(g_\theta(z_i)) 
\end{eqnarray}
which is  dependent on the sample index $i$.
Therefore, aside from the geometric scaling factors, the main difference  comes from the  separating hyperplane for linear classifiers.
% within each min-batch  training, the regression error in the feature space is averaged
%before being back-propagated with SGD.
More specifically,  the hyperplane for geometric GAN is obtained for   samples within each mini-batch, while
the classifier for  regression is optimally designed  for each pair of samples.

\section{Proof for Theorem~\ref{thm:conv0}}

The proof technique is inspired from that of EB-GAN \cite{zhao2017ebgan}.
We first need the following two lemmas as the extensions of Lemma 1 in \cite{zhao2017ebgan}.

\begin{lemma}\label{lem:ineq}
Let  $\varphi(y) = (m-y) +  [m+y]_+$.  The minimum of $\varphi(y)$ is $2m$ and is
reached at  all $y\geq-m$.
\end{lemma}
\begin{proof}
If $y\geq -m$, then $\varphi(y) = 2m$.
If $y\leq -m$, the $\varphi(y)= -2y$, whose minimum $2m$ is achieved 
at $y=-m$.
\end{proof}

%The next lemma is a direct extension of Lemma~\ref{lem:ineq}.
\begin{lemma}\label{lem:minprop}
For given \(\alpha, \beta \geq 0\), The minimum of \(\varphi(y)=\alpha \left[ m - y \right]_{+} + \beta \left[ m + y \right]_{+}\) exist if \(y \in \left[-m, m\right]\). More specifically, the minimum of \(\varphi(y)\) is \(2 \beta m\) at \(y=m\) if \(\alpha > \beta\), or \(2 \alpha m\) at \(y=-m\) if \(\alpha \leq \beta\).
\end{lemma}

\begin{proof}
If \( y \geq m \), then \( \varphi(y) = \beta \left[ m + y \right]_{+} = \beta (m + y) \geq 2 \beta m \). Thus, \( \inf_{y \in [m,\infty)} {\varphi(y)} = 2 \beta m \) at \( y = m \). Similarly, if \( y \leq -m \), then \( \inf_{y \in (-\infty, -m]} {\varphi(y)} = 2 \alpha m \) at \( y = -m \). 
For \( y \in [-m, m] \), \( \varphi(y) = \alpha (m-y) + \beta (m+y) = (\alpha + \beta)m + (\beta - \alpha)y \). If \( \alpha > \beta \), \( \inf_{y \in [-m,m]} {\varphi(y)} = 2 \beta m \) at \( y = m \) since \( \varphi (y) \) is a decreasing function on \(y \in [-m, m]\). Similarly, if \( \alpha \leq \beta \), \( \inf_{y \in [-m,m]} {\varphi(y)} = 2 \alpha m \) at \( y=-m \) since \( \varphi (y) \) is increasing at \(y \in [-m, m]\).
\end{proof}

Now we are ready for the proof.
\begin{proof}[Proof of Theorem~\ref{thm:conv0}]
Since $R(D,g)$ and $L(D,g)$ are lower semi-continuous functions,  \( R(D^*, g^*) \) has a finite value for the optimal solution \( (D^*, g^*) \).
Moreover, due to the alternating minimization,
 the pair satisfies:
\begin{eqnarray}
R(D^*, g^*) \leq R(D, g^*),\quad \forall D \label{eq:Ropt} \\
L(D^*, g^*) \leq L(D^*, g),\quad \forall g \label{eq:Dopt}
\end{eqnarray}
%First, {we show that $R(D^*, g^*) \leq 2$.
First, define the set $$A = \{x | p_{x}(x) \leq p_{g^*}(x) \}$$
and  observe that
%$$ R(D,g^*)  =\int {p_{x}(x)     \big[ 1 - D(x) \big]_{+}  +  p_{g^*_\theta}(x) \big[ 1 + D(x) \big]_{+}  dx}.$$
%
%More specifically,
\begin{eqnarray}
    R(D,g^*) & =&\int {p_{x}(x)     \big[ 1 - D_{w,b,\zeta}(x) \big]_{+}  +  p_{g^*_\theta}(x) \big[ 1 + D_{w,b,\zeta}(x) \big]_{+}  dx} \notag \\
%    & \quad + \int {p_{g_\theta}(x) \big[ 1 + D_{w,b,\zeta}(x) \big]_{+} dx} \notag \\
    & = &      \int {\mathbbm{1}_A(x) p_{x}(x)     \big[ 1 - D_{w,b,\zeta}(x) \big]_{+} dx} \notag \\
        & & + \int {\mathbbm{1}_A(x) p_{g^*_\theta}(x) \big[ 1 + D_{w,b,\zeta}(x) \big]_{+} dx} \notag \\
    &  & + \int {\mathbbm{1}_{A^c}(x)    p_{x}(x)     \big[ 1 - D_{w,b,\zeta}(x) \big]_{+} dx} \notag\\
        & & + \int {\mathbbm{1}_{A^c}(x)   p_{g^*_\theta}(x) \big[ 1 + D_{w,b,\zeta}(x) \big]_{+} dx} \notag
\end{eqnarray}
where  $A^c$ denotes the complementary set and $\mathbbm{1}_A(x)$ is an indicator function 
$$ \mathbbm{1}_A(x)= \begin{cases} 1, & x\in A \\ 0, & x \notin A \end{cases}. $$

From Lemma~\ref{lem:minprop}, we know that 1)  when $p_x(x) < p_{g^*_\theta}(x)$, the term within the integral
achieves its minimum value of $2 p_x(x)$ at $D^*(x)=-1$, or 2) when $p_x(x) > p_{g^*_\theta}(x)$, the term within the integral
achieves its minimum value of $2p_{g^*_\theta}(x)$ at $D^*(x)=1$. Therefore,
\begin{eqnarray}
    R(D^*,g^*)
%     & =&    & & + \int {\mathbbm{1}_{p_x(x) > p_{g^*}(x)}    p_{g^*_\theta}(x) \big[ 1 + D_{w,b,\zeta}(x) \big]_{+} dx} \notag \\
%    & =&       \int {\mathbbm{1}_{p_{x}(x) \leq p_{g^*}(x)} p_{x}(x)     \big[ 1 - D_{w,b,\zeta}(x) \big]_{+} dx} \notag \\
%    & &+ \int {\mathbbm{1}_{p_{data(x)} > p_{g^*}(x)}    p_{g^*_\theta}(x) \big[ 1 + D_{w,b,\zeta}(x) \big]_{+} dx} \notag \\
    & =&       2 \int {\mathbbm{1}_A(x) p_{x}(x)     dx} + 2 \int {\mathbbm{1}_{A^c}(x)   p_{g^*_\theta}(x) dx} \notag \\
    & =&       2 \int { \big[ \mathbbm{1}_A(x)p_{x}(x) + \big(1-\mathbbm{1}_A(x)\big) p_{g^*_\theta}(x) \big] dx} \notag \\
    & = &      2 + 2 \int {\mathbbm{1}_A(x) \big( p_{x}(x) - p_{g^*_\theta}(x) \big) dx} \notag \\
    & \leq & 2  \label{eq:upper}
\end{eqnarray}
where the last inequality comes from
$p_{x}(x) - p_{g_\theta}(x) \leq 0$ for all $x \in A$.

Second, we will show that $R(D^*, g^*) \geq 2$.
Because \eqref{eq:Dopt} holds for arbitrary pdf $p(x)$, we have
\begin{eqnarray*}
\int p_{g_\theta^*}(x)\left(-D^*(x)\right) dx &\leq &  \int p_{x}(x)\left(-D^*(x)\right) dx
\end{eqnarray*}
By adding 1 to both sides, we have
\begin{eqnarray*}
\int p_{g_\theta^*}(x)\left(1-D^*(x)\right) dx &\leq &  \int p_{x}(x)\left(1-D^*(x)\right) dx \\
&\leq&  \int p_{x}(x)\left[1-D^*(x)\right]_+ dx \
\end{eqnarray*}
where the last inequality  comes from $x \leq [x]_+=\max\{0,x\}$.
%
%%P
%%The optimal solution, i.e. Nash equilibrium, \( (D^*, g^*) \) satisfies equation 26. Thus, if we put a generator \(g_{data}\) such that follows true data distribution \(p_{data}\) into the right hand side at equation 26, then 
%$$
%\int {p_{g_\theta^*}(x)\big[1-D^*(x)\big]_{+}} dx \leq \int {p_{x}(x)\big[1-D^*(x)\big]_{+}} dx
%$$
%by replacing $p_g(x)$ with the pdf $p_x(x)$ for the true samples. 
Now, by adding \( \int {p_{g_\theta^*}(x)\big[1+D^*(x)\big]_{+}} dx \) on both sides, we have:
\begin{eqnarray}
& \int {p_{g_\theta^*}(x)\big(1-D^*(x)\big)} dx + \int {p_{g_\theta^*}(x)\big[1+D^*(x)\big]_{+}} dx \notag\\
& \quad \quad \quad \leq \int {p_{x}(x)\big[1-D^*(x)\big]_{+}} dx + \int {p_{g_\theta^*}(x)\big[1+D^*(x)\big]_{+}} dx = R(D^*, g^*)
\end{eqnarray}

From Lemma~\ref{lem:ineq}, we know that $\big(1-D^*(x)\big)+\big[1+D^*(x)\big]_{+} \geq 2$. Thus, 
we have
\begin{eqnarray*}
R(D^*, g^*) &\geq &  \int {p_{g_\theta^*}(x)\big(1-D^*(x)\big)} dx + \int {p_{g_\theta^*}(x)\big[1+D^*(x)\big]_{+}} dx\\
&\geq &  \int p_{g_\theta^*}(x) 2 dx 
\quad \geq \quad  2
\end{eqnarray*}
Thus, 
\begin{eqnarray}
    2 \leq R(D^*, g^*) \leq 2
\end{eqnarray}

Finally, the equality in \eqref{eq:upper} holds if and only if 
\begin{align*}
  \int { \mathbbm{1}_{A}(x) \big( p_{x}(x) - p_{g^*_\theta}(x) \big) dx} = 0 
%  \\
%  \int { \mathbbm{1}_{A^c}(x) \big( p_{g^*_\theta}(x) - p_{x}(x) \big) dx} = 0 
\end{align*}
The above equalities hold if and only if \( p_{g*}(x) = p_{data}(x) \) almost everywhere \cite{zhao2017ebgan}. 
This concludes the proof.

\end{proof}

\end{document}